\documentclass[twoside]{article}
\usepackage{arxiv}

\usepackage[authoryear]{natbib}

\usepackage{xcolor}
\definecolor{niceRed}{RGB}{190,38,38}
\definecolor{royalBlue}{HTML}{057DCD}

\usepackage{graphicx}

\usepackage{amsmath,amsfonts,amssymb,amsthm,bbm}

\usepackage[pagebackref]{hyperref}
\hypersetup{
     colorlinks=true,
     linkcolor=royalBlue,
     filecolor=royalBlue,
     citecolor =niceRed,      
     urlcolor=magenta,
     }
\usepackage{algorithm}
\usepackage{verbatim}
\usepackage[noend]{algpseudocode}
\newenvironment{sproof}{%
  \proof}{\endproof}
\usepackage{appendix}

\def\Scal{\mathcal{S}}
\def\S{\mathfrak{S}}
\def\M{\mathcal{M}}
\def\A{\mathcal{A}}
\def\ep{\epsilon}
\def\Nn{\mathbb{N}}
\def\a{\alpha}
\def\ind{\mathbbm{1}}

\def\Set{S}
\def\Sets{(\Set_1,\dots,\Set_r)}

\def\subse{\subseteq}
\def\pr{\textnormal{Pr}}

\def\central{\pi_0}
\def\b{\beta}
\def\Mal{\M_{\central,\b}}
\def\Malin[#1,#2]{\M_{#1,#2}}

\def\MadvS{\Mal^{\Scal}}

\def\prof{\Pi}
\def\samples{(\pi_1,\dots,\pi_r)}
\def\sample{\pi}

\def\kt{d_{KT}}

\def\positional{\hat{\pi}}
\def\mle{\sample^*}
\def\mleA{\mle_\mathcal{A}}
\def\mlenat{\sample^\circ}

\def\num[#1,#2]{W_{#1#2}}
\def\mod{\mathop{\mathrm{mod}}}
\def\sym{\text{Sym}}
\def\fprof{\hat{\prof}}

\def\estim{\tilde{\sample}}
\def\F{\mathcal{F}}
\def\id{\text{id}}

\def\M{\mathcal{M}}
\def\N{\mathcal{N}}
\def\eps{\varepsilon}
\def\D{\mathcal{D}}

\def\posest{\textsc{PosEst}}

\def\St{\mathcal{D}}

\newtheorem{theorem}{Theorem}[section]
\newtheorem{lemma}{Lemma}[section]

\newtheorem{remark}{Remark}[section]
\newtheorem{cor}{Corollary}[section]

\def\numle[#1,#2]{n_{#1\succ #2}}
\def\be{\textnormal{Be}}
\def\bin{\textnormal{Bin}}
\def\kl{\textnormal{D}_{\textnormal{KL}}}
\def\I{\mathcal{I}}
\def\r{|\Scal|}

\title{Aggregating Incomplete and Noisy Rankings}

\author{\large{Dimitris Fotakis}\thanks{National Technical University of Athens, mail: \color{magenta} fotakis@cs.ntua.gr \color{black}} \And \large{Alkis Kalavasis}\thanks{National Technical University of Athens, mail: \color{magenta} kalavasisalkis@mail.ntua.gr \color{black}} \And \large{Konstantinos Stavropoulos}\thanks{National Technical University of Athens, mail: \color{magenta} kons.stavropoulos@gmail.com \color{black}}}

\begin{document}
\footnotetext{Accepted at the 24th International Conference on
Artificial Intelligence and Statistics (AISTATS) 2021.}

\maketitle

\begin{abstract}
We consider the problem of learning the true ordering of a set of alternatives from largely incomplete and noisy rankings. We introduce a natural generalization of both the classical Mallows model of ranking distributions and the extensively studied model of noisy pairwise comparisons. Our \emph{selective Mallows model} outputs a noisy ranking on any given subset of alternatives, based on an underlying Mallows distribution. Assuming a sequence of subsets where each pair of alternatives appears frequently enough, we obtain strong asymptotically tight upper and lower bounds on the sample complexity of learning the underlying complete ranking and the (identities and the) ranking of the top-$k$ alternatives from selective Mallows rankings. Moreover, building on the work of (Braverman and Mossel, 2009), we show how to efficiently compute the maximum likelihood complete ranking from selective Mallows rankings.
\end{abstract}

\section{Introduction}
\label{s:intro}

Aggregating a collection of (possibly noisy and incomplete) ranked preferences into a complete ranking over a set of alternatives is a fundamental and extensively studied problem with numerous applications. Ranking aggregation has received considerable research attention in several fields, for decades and from virtually all possible aspects. 

Most relevant, Statistics investigates the properties of \emph{ranking distributions}, which provide principled ways to generate noisy rankings from structural information about the alternatives' relative order. Best known among them are the distance-based model of \citet{mallows1957non} and the parametric models of \citet{thurstone1927law, smith1950discussion, bradley1952rank, plackett1975analysis} and \citet{luce2012individual}. Moreover, Machine Learning and Statistical Learning Theory aim to develop (statistically and computationally) efficient ways of retrieving the true ordering of the alternatives from noisy (and possibly incomplete) rankings (see e.g., 
\citep{xia2019learning} and the references therein). 

Virtually all previous work in the latter research direction assumes that the input is a collection of either complete rankings (chosen adversarially, e.g.,  \cite{ailon2008aggregating,MS07}, or drawn from an unknown ranking distribution, e.g., \citep{caragiannis2013noisy,busa2019optimal}), or outcomes of noisy pairwise comparisons (see e.g., \citep{feige1994computing,mao2018minimax}). Due to a significant volume of relatively recent research, the computational and statistical complexity of determining the best ranking based on either complete rankings or pairwise comparisons are well understood.

However, in most modern applications of ranking aggregation, the input consists of incomplete rankings of more than two alternatives. E.g., think of e-commerce or media streaming services, with a huge collection of alternatives, which generate personalized recommendations based on rankings aggregated by user ratings (see also \cite{hajek2014minimax}). Most users are able to rank (by rating or reviewing) several alternatives, definitely much more than two, but it is not even a remote possibility that a user is familiar with the entire inventory (see also \citep{MCE16} for applications of incomplete rankings to ranked preference aggregation, and \citep{YildizDEKOCCI20} on why incomplete rankings are preferable in practice). 

Motivated by the virtual impossibility of having access to complete rankings in modern applications, we introduce the \emph{selective Mallows model}, generalizing both the classical Mallows model of ranking distributions and the extensively studied model of noisy pairwise comparisons. Under the selective Mallows model, we investigate the statistical complexity of learning the central ranking and the (identities and the) ranking of the top-$k$ alternatives, and the computational complexity of maximum likelihood estimation.

\subsection{The Selective Mallows Model}
\label{s:model}

The \emph{Mallows model} \citep{mallows1957non} is a fundamental and extensively studied family of ranking distributions over the symmetric group $\S_n$. A \emph{Mallows distribution} $\Mal$ on a set of $n$ alternatives is parameterized by the \emph{central ranking} $\pi_{0} \in \S_n$ and the \emph{spread parameter} $\beta > 0$. The probability of observing a ranking $\pi \in \S_n$ is proportional to $\exp(-\beta d(\pi_{0}, \pi))$, where $d$ is a notion of ranking distance of $\pi_0$ to $\pi$. In this work, we consider the number of discordant pairs, a.k.a. the Kendall tau distance, defined as $\kt(\pi_0, \pi) = \sum_{i < j} \ind\left\{(\pi_0(i) - \pi_0(j))(\pi(i) - \pi(j)) < 0 \right\}$. 

The problem of aggregating a collection $\pi_1, \ldots, \pi_r \in \S_n$ of complete rankings asks for the \emph{median} ranking 
$\pi^\star = \arg\min_{\sigma \in \S_n} \sum_{j=1}^r \kt(\sigma, \pi_j)$.
Computing the median is equivalent to a weighted version of Feedback Arc Set in tournaments, which is NP-hard \citep{ailon2008aggregating} and admits a polynomial-time approximation scheme \citep{MS07}. If the rankings are independent samples from a Mallows distribution, the median coincides with the \emph{maximum likelihood ranking} and can be computed efficiently with high probability \citep{braverman2009sorting}.

The \emph{selective Mallows model} provides a principled way of generating noisy rankings over any given subset of alternatives, based on an underlying Mallows distribution $\Mal$. Given a central ranking $\central\in\S_n$, a spread parameter $\b>0$ and a selection sequence $\Scal=\Sets$, where each $ \Set_\ell \subseteq [n]$, the \textit{selective Mallows distribution} $\MadvS$ assigns a probability of 
\[
    \pr[\samples|\central,\b,\Scal] = 
    \prod_{\ell\in [r]}\frac{1}{Z(S_\ell,\beta)}e^{-\b \kt(\central,\sample_\ell)}\,,
\]
to each incomplete ranking profile $\samples$. In the probability above, each $\sample_\ell$ is a permutation of $\Set_\ell$, $\kt(\central,\sample_\ell)$ is the number of pairs in $S_i$ ranked reversely in $\central$ and $\sample_\ell$ (which naturally generalizes the Kendall tau distance to incomplete rankings), and the normalization constants $Z(\Set_\ell, \beta)$ correspond to a Mallows distribution on alternatives $\Set_\ell$ and depend only on $|\Set_\ell|$ and $\beta$. The samples $\pi_1, \ldots, \pi_r$ are independent, conditioned on the selection sequence $\Scal$. We refer to $\prof = \samples$ as a sample profile of length $r$. 

In a selective Mallows sample $\pi_\ell$, the probability that two alternatives in $S_\ell$ are not ranked as in $\pi_0$ depends on their distance in the restriction of $\pi_0$ to $S_\ell$ (instead of their original distance in $\pi_0$). E.g., if $\pi_0$ is the identity permutation and $S_\ell = \{1, n\}$, the probability that $\pi_\ell = (1, n)$ (resp. $\pi'_\ell = (n, 1)$) is the $\ell$-th sample in $\prof$ is $1/(1+e^{-\beta})$ (resp. $e^{-\beta}/(1+e^{-\beta})$). Hence, the selective Mallows model generalizes both the standard Mallows model (if each $S_\ell = [n]$) and the model of noisy pairwise comparisons (if $\Scal$ consists entirely of sets $S_\ell$ with $|S_\ell|=2$). Moreover, using sets $S_\ell$ of different cardinality, one can smoothly interpolate between complete rankings and pairwise comparisons. 

The amount of information provided by a selective Mallows model $\MadvS$ about $\pi_0$ is quantified by how frequently different pairs of alternatives compete against each other in $\prof$. We say that a selective Mallows model $\MadvS$ is \emph{$p$-frequent}, for some $p \in (0, 1]$, if every pair of alternatives appears in at least a $p$ fraction of the sets in $\Scal$ (we assume that each pair appears together at least once in $\Scal$). E.g., for $p = 1$, we recover the standard Mallows model, while $p \approx 2/n^2$ corresponds to pairwise comparisons. The definition of ($p$-frequent) selective Mallows model can be naturally generalized to unbounded selection sequences $\Scal$, which however is beyond the scope of this work. 

In this work, we investigate the statistical complexity of retrieving either the central ranking $\pi_0$ or its top-$k$ ranking from $p$-frequent selective Mallows samples, and the computational complexity of finding a maximum likelihood ranking from a fixed number $r$ of $p$-frequent samples. In \emph{learning from incomplete rankings}, for any given $p, \beta, \eps > 0$, we aim to upper and lower bound the least number of samples $r^\star(p, \beta, \eps)$ (resp.  $r^\star_k(p, \beta, \eps)$) from a selective Mallows distribution $\MadvS$ required to learn $\pi_0$ (resp. the top-$k$ ranking of $\pi_0$) with probability at least $1-\eps$, where $\Scal$ is any $p$-frequent selection sequence. In \emph{maximum likelihood estimation}, for any given $p, \beta, \eps > 0$, given a sample profile $\prof$ of length $r$ from a $p$-frequent selective Mallows distribution $\MadvS$, we aim to efficiently compute either a ranking that is at least as likely as $\pi_0$, or even a maximum likelihood ranking $\pi^\star$. The interesting regime for maximum likelihood estimation is when $r$ is significantly smaller than $r^\star(p, \beta, \eps)$.

We shall note here that the $p$-frequent condition can be replaced by a milder 
one, where each selection set is drawn independently from a given distribution over the subsets of $[n]$, such that the probability that any specific pair of alternatives appears in a sampled set is at least $p$. Although we focus (for  simplicity) on the deterministic $p$-frequent assumption, we expect similar results to hold for the randomized case. For a detailed discussion of the randomized $p$-frequent assumption, we refer the reader to the Appendix~\ref{appendix:rand}.

\subsection{Contribution}
\label{s:contrib}

On the conceptual side, we introduce the selective Mallows model, which allows for a smooth interpolation between learning from noisy complete rankings and sorting from noisy pairwise comparisons. On the technical side, we practically settle the statistical complexity of learning the central ranking and the top-$k$ ranking of a $p$-frequent selective Mallows model. Moreover, we show how to efficiently compute a maximum likelihood ranking from $r$ selective samples.

We believe that a significant advantage of our work lies in the simplicity and the uniformity of our approach. Specifically, all our upper bounds are based on the so-called \emph{positional estimator} (Algorithm~\ref{algo:positional}), which ranks an alternative $i$ before any other alternative $j$ ranked after $i$ in the majority of the samples. The positional estimator belongs to the class of pairwise majority consistent rules \citep{caragiannis2013noisy}.

Generalizing the (result and the) approach of \citet[Theorem~3.6]{caragiannis2013noisy}, we show (Theorem~\ref{thm:upper_central}) that if $pr = O(\frac{\log(n/\eps)}{(1-e^{-\beta})^2})$, the central ranking of a $p$-frequent selective Mallows model can be recovered with probability at least $1-\eps$. Namely, observing a logarithmic number of noisy comparisons per pair of alternatives suffices for determining their true order. Theorem~\ref{thm:upper_central} generalizes (and essentially matches) the best known bounds on the number of (passively chosen%
\footnote{If the algorithm can actively select which pairs of alternatives to compare, $O(n\log n)$ noisy comparisons suffice for sorting, e.g., \cite{braverman2008noisy,feige1994computing}.})
comparisons required for sorting \citep{mao2018minimax}. 

Interestingly, we show that the above upper bound is practically tight. Specifically, Theorem~\ref{thm:lower_central} shows that for any $p \in (0, 1/2]$, unless $rp = \Omega(\log(n/\eps)/\beta)$ noisy comparisons per pair of alternatives are observed, any estimator of the central ranking from $p$-frequent selective samples fails to recover $\pi_0$ with probability larger than $\eps$. Hence, observing incomplete rankings with (possibly much) more than two alternatives may help in terms of the number of samples, but it does not improve the number of noisy comparisons per pair required to recover the true ordering of the alternatives. 

In Section~\ref{s:approx}, we generalize the proof of Theorem~\ref{thm:approx} and show that the positional estimator smoothly (and uniformly wrt different alternatives) converges to the central ranking $\pi_0$ of a $p$-frequent selective Mallows model $\MadvS$, as the number of samples $r$ (and the number of noisy comparisons $pr$ per pair) increase. Specifically, Theorem~\ref{thm:approx} shows that the positional estimator has a remarkable property of the \emph{average position estimator} \citep[Lemma~18]{braverman2008noisy}\,: as $r$ increases, the position of any alternative $i$ in the estimated ranking converges fast to $\pi_0(i)$, with high probability. Since we cannot use the average position estimator, due to our incomplete rankings, where the positions of the alternatives are necessarily relative, we need to extend \citep[Lemma~18]{braverman2008noisy} to the positional estimator. 

Combining Theorem~\ref{thm:upper_central} and Theorem~\ref{thm:approx}, we show  (Section~\ref{sec:topk}) that for any $k = \Omega(1/(p\beta))$, we can recover the identities and the true ranking of the top-$k$ alternatives in $\pi_0$, with probability at least $1-\eps$, given $r = O(\frac{\log(k/\eps)}{p(1-e^{-\beta})^2}+\frac{\log(n/\eps)}{p^2\beta k})$ $p$-frequent selective samples. The second term accounts for learning the identities of the top $O(k)$ elements in $\pi_0$ (Theorem~\ref{thm:approx}), while the first term accounts for learning the true ranking of these $O(k)$ elements (Theorem~\ref{thm:upper_central}). For sufficiently large $k$, the first term becomes dominant. Applying the approach of Theorem~\ref{thm:lower_central}, we show that such a sample complexity is practically best possible.

Moreover, building on the approach of \citet[Lemma~18]{braverman2008noisy} and exploiting Theorem~\ref{thm:approx}, we show how to compute a maximum likelihood ranking (resp. a ranking that is at least as likely as $\pi_0$), given $r$ samples of a $p$-frequent selective Mallows distribution $\MadvS$, in time roughly $n^{O(1/(r\beta p^4))}$ (resp. $n^{O(1/(r\beta p^2))}$), with high probability (see also Theorem~\ref{thm:mle_alg} for the exact running time). The interesting regime for maximum likelihood estimation is when $r$ is much smaller than the sample complexity of learning $\pi_0$ in Theorem~\ref{thm:upper_central}. Our result compares favorably against the results of \citet{braverman2008noisy} if $pr$ is small. E.g., consider the extreme case where $pr = 1$ (i.e., each pair is compared once in $\prof$). Then, for small values of $\beta$, the running time of \citep[Theorem~8]{braverman2008noisy} becomes $n^{O(1/\beta^4)}$, while the running time of maximum likelihood estimation from $p$-frequent selective samples becomes $n^{O(1/(p^3 \beta))}$. Thus, large incomplete rankings mitigate the difficulty of maximum likelihood estimation (compared against noisy pairwise comparisons), if $rp = \Theta(1)$, $\beta$ is small, and $1/\beta$ is much smaller than $1/p$. 

In the following, we provide the intuition and proof sketches for our main results. For the full proofs of all our technical claims, we refer the reader to the Appendix~\ref{appendix:start}.

\noindent{\bf Notation.}
We conclude this section with some additional notation required in the technical part of the paper. For any ranking $\pi$ of some $S\subse[n]$ and any pair of alternatives $i, j \in S$, we let $i \succ_{\pi} j$ denote that $i$ precedes $j$ in $\pi$, i.e., that $\pi(i) < \pi(j)$. When we use the term reduced central ranking (according to a sample), we refer to the permutation of the elements of some selection set according to the central ranking. For any object $B$, we use the notation $B=B[\prof]$ to denote that $B$ depends on a sample profile $\prof$. Moreover, for simplicity and brevity, we use the asymptotic notation $O_\beta$ (or $\Omega_\beta$) to hide polynomial terms in $1/\beta$. 

\subsection{Related Work}
\label{s:related}

There has been a huge volume of research work on statistical models over rankings (see e.g., \cite{fligner1993probability, marden1996analyzing, xia2019learning} and the references therein). The \citet{mallows1957non} model plays a central role in the aforementioned literature. A significant part of this work concerns either extensions and generalizations of the Mallows model (see e.g., \citep{fligner1986distance,murphy2003mixtures,lebanon2003conditional}, and also \citep{lebanon2008non,lu2014effective,busa2014preference} more closely related to partial rankings) or statistically and computationally efficient methods for recovering the parameters of Mallows distributions (see e.g., \citep{adkins1998non,caragiannis2013noisy,liu2018efficiently,busa2019optimal}). 

From a conceptual viewpoint, the work of \citet{hajek2014minimax} is closest to ours. \cite{hajek2014minimax} introduce a model of selective incomplete Thurstone and Plackett-Luce rankings, where the selection sequence consists of sets of $k$ alternatives selected uniformly at random. They provide upper and lower bounds on how fast optimizing the log-likelihood function from incomplete rankings (which in their case is concave in the parameters of the model and can be optimized via e.g., gradient descent) converges to the model's true parameters. In our case, however, computing a maximum likelihood ranking is not convex and cannot be tackled with convex optimization methods. From a technical viewpoint, our work builds on previous work by \citet{braverman2008noisy}, \citet{caragiannis2013noisy} and \citet{busa2019optimal}. 

For almost three decades, there has been a significant interest in ranked preference aggregation and sorting from noisy pairwise comparisons. One branch of this research direction assumes that the algorithm actively selects which pair of alternatives to compare in each step and aims to minimize the number of comparisons required for sorting (see e.g., \citep{feige1994computing,braverman2008noisy}, or \citep{ailon2012active} for sorting with few errors, or \citep{braverman2016parallel} for parallel algorithms). A second branch, closest to our work, studies how many passively (see e.g., \citep{mao2018breaking,mao2018minimax}) or randomly (see e.g., \citep{wauthier2013efficient}) selected noisy comparisons are required for ranked preference aggregation and sorting. A more general problem concerns the design of efficient approximation algorithms (based on either sorting algorithms or common voting rules) for aggregating certain types of incomplete rankings, such as top-$k$ rankings, into a complete ranking (see e.g., \cite{Ailon10,MM20}). Moreover, there has been recent work on assigning ranking scores to the alternatives based on the results of noisy pairwise comparisons, by likelihood maximization through either gradient descent or majorize-maximization (MM) methods (see e.g., \citet{VojnovicYZ20}). Such works on learning from pairwise comparisons are also closely related to our work from a graph-theoretic viewpoint, since they naturally correspond to weighted graph topologies, whose properties (e.g., Fiedler eigenvalue of the comparison matrix~\citep{hajek2014minimax, shah2016estimation, khetan2016data, vojnovic2016parameter, negahban2017rank, VojnovicYZ20} or degree sequence~\citep{pananjady2020worst}) characterize the sample complexity and convergence rate of various learning approaches. The comparison graph of $p$-frequent Mallows is the clique. 

Another related line of research (which goes back at least to \citet{ConitzerS05}) investigates how well popular voting rules (e.g., Borda count, Kemeny ranking, approval voting) behave as maximum likelihood estimators for either the complete central ranking of the alternatives, or the identities of the top-$k$ alternatives, or the top alternative (a.k.a. the winner). In this line of work, the input may consist of complete or incomplete noisy rankings \citep{xia2011determining,ProcacciaRS12}, the results of noisy pairwise comparisons \citep{shah2017simple}, or noisy $k$-approval votes \citep{caragiannis2017learning}.


\section{Retrieving the Central Ranking}
\label{s:central}

In this section, we settle the sample complexity of learning the central ranking $\pi_0$ under the selective Mallows model. We show that a practically optimal strategy is to neglect the concentration of alternatives' positions around their initial positions in $\pi_0$ and act as if the samples are a set of pairwise comparisons with common probability of error only depending on $\b$.

\noindent{\bf Positional Estimator.}
Given a sample profile $\prof=\samples$ corresponding to a selection sequence $\Scal=\Sets$, we denote with $\positional=\positional[\prof]$ the permutation of $[n]$ output by Algorithm \ref{algo:positional}.

\begin{algorithm}
\caption{Positional Estimator of profile $\prof$} \label{algo:positional}
\begin{algorithmic}[1]
\Procedure{PosEst}{$\Pi$} 
\State $\positional \gets \textbf{0}_n$ 
\For{$i \in [n]$}
    \For{$j \in [n]$} 
        \State $\Pi_{i,j} \gets \{\sample \in \prof : i,j \in \sample \}$ 
        \State $n_{j \succ i} \gets |\{ \sample \in \Pi_{i,j} : j \succ_{\sample} i\}|$
        \If{$n_{j \succ i} \geq |\Pi_{i,j}|/2$}
            \State $\positional(i) \gets \positional(i) + 1$
        \EndIf
    \EndFor    
\EndFor 
\State Break ties in $\positional$ uniformly at random\\
\Return $\positional$
\EndProcedure
\end{algorithmic}
\end{algorithm}

Algorithm \ref{algo:positional} first calculates the pairwise majority position of each alternative, by comparing each  $i \in [n]$ with any other $j \in [n]$ in the joint subspace of the sample profile where $i$ and $j$ appear together. Intuitively, $\positional(i)$ equals  $|\{j:j\text{ ranked before } i\text{ most times}\}|$.
The algorithm, after breaking  ties uniformly at random, outputs $\positional$. We call $\positional$ the \emph{positional estimator} (of the sample profile $\Pi$), which belongs to the class of  pairwise majority consistent rules, introduced by \citet{caragiannis2013noisy}.

\noindent{\bf Sample Complexity.} Motivated by the algorithm of \citet{caragiannis2013noisy} for retrieving the central ranking from complete rankings, we utilize the {\posest} (Algorithm \ref{algo:positional}) and establish an upper bound on the sample complexity of learning the central ranking in the selective Mallows case.
\begin{theorem}\label{thm:upper_central}
For any $\ep\in(0,1)$, $\central\in\S_n$, $\b>0$, $p\in(0,1]$, there exists an algorithm such that, given a sample profile from $\MadvS$, where $\Scal$ is a $p$-frequent selection sequence of length $O(\frac{1}{p(1-e^{-\b})^2}\log(n/\ep))$, Algorithm~\ref{algo:positional} retrieves $\central$ with probability at least $1-\ep$.
\end{theorem}

\begin{sproof}
If we have enough samples so that every pair of alternatives is ranked correctly in the majority of its comparisons, with probability at least $1-\ep/n^2$, then, by union bound, all pairs are ranked correctly in the majority of their comparisons with probability at least $1-\ep$, which, in turn, would imply the theorem. If the number of samples is $r$, then each pair is compared at least $pr$ times in the sample. The Hoeffding bound implies that the probability that a pair is swapped in the majority of its appearances decays exponentially with $pr$.
\end{sproof}

For the complete proof, we refer the reader to the Appendix \ref{s:pf1}.

In fact, the positional estimator is an optimal strategy with respect to the sample complexity of retrieving the central ranking. This stems from the fact that if for some pair, the total number of its comparisons in the sample is small, then there exists a family of possible central rankings where different alternatives cannot be easily ranked, due to lack of information.
We continue with an essentially matching lower bound:

\begin{theorem}\label{thm:lower_central}
For any $p\in(0,1]$, $\ep\in(0,1/2]$, $\b>0$ and $r=o(\frac{1}{\b p}\log(n/\ep))$ there exists a $p$-frequent selection sequence $\Scal$ with $|\Scal|=r$, such that for any central ranking estimator, there exists a $\central\in\S_n$, such that the estimator, given a sample profile from $\MadvS$, fails to retrieve $\central$ with probability more than $\ep$.
\end{theorem}

\begin{sproof}
Let $\Scal$ contain $p|\Scal|$ sets with all alternatives and $(1-p)|\Scal|$ sets of size at most $n\sqrt{p/(1-p)}$. For any $i,j\in[n]$, let $\num[i,j](\Scal)$ be the number of sets of $\Scal$ containing both $i$ and $j$, that is, the number of the appearances of pair $(i,j)$. Clearly, $\Scal$ is $p$-frequent and: 
\begin{equation}\label{eq:bounded_appearances}
    \sum_{i<j}\num[i,j](\Scal)\le pn^2|\Scal|.
\end{equation} 
Assume that $|\Scal|<\frac{1}{8p\b}\log(\frac{n(1-\ep)}{4\ep})$. 

We will show that there exists a set of $n/2$ disjoint pairs of alternatives which we observe only a few times in the samples. Assume that $n$ is even. We design a family $\{P_t\}_{t\in[n/2]}$ of perfect matchings on the set of alternatives $[n]$. Specifically, we consider $n/2$ sets of $n/2$ disjoint pairs $P_1 = \{(1,2),(3,4),\dots,(n-1,n)\}$, $P_2 = \{(1,4),(3,6),\dots,(n-1,2)\}$ and, in general, 
$P_t = \{(1,(2t)\mod n),\dots, (n-1,(2t+n-2)\mod n)\}$ for $t \in [n/2].$

Observe that no pair of alternatives appears in more than one perfect matching of the above family. Therefore:
\begin{equation}\label{eq:bounded_appearances_matchings}
    \sum_{t\in[n/2]}\sum_{(i,j)\in P_t}\num[i,j](\Scal)\le\sum_{i<j}\num[i,j](\Scal)\,.
\end{equation}
Combining \eqref{eq:bounded_appearances}, the bound for $|\Scal|$ and \eqref{eq:bounded_appearances_matchings}, we get that:
\begin{equation*}\label{eq:bounded_appearances_matching}
    \exists t\in[n/2]: \sum_{(i,j)\in P_t}\num[i,j](\Scal) < \frac{n}{4\b}\log\left(\frac{n(1-\ep)}{4\ep}\right)\,.
\end{equation*}
Hence, since $|P_t|=[n/2]$, there exist at least $n/4$ pairs $(i,j)\in P_t$ with $\num[i,j](\Scal) < \frac{1}{\b}\log\left(\frac{n(1-\ep)}{4\ep}\right)$.

We conclude the proof with an information-theoretic argument based on the observation that if the pairs of $P_t$, $n/4$ of which are observed few times, are adjacent in the central ranking, then the probability of swap is maximized for each pair. Moreover, the knowledge of the relative order of the elements in some pairs in the matching does not provide any information about the relative order of the elements in any of the remaining pairs. Intuitively, since each of $n/4$ pairs is observed only a few times, no central ranking estimator can be confident enough about the relative order of the elements in all these pairs.
\end{sproof}

The complete proof can be found at the Appendix \ref{s:pf2}.


\section{Approximating the Central Ranking with Few Samples}
\label{s:approx}

We show next that the positional estimator smoothly approximates the position of each alternative in the central ranking, within an additive term that diminishes as the number of samples $r$ increases.

The average position estimation approximates the positions of the alternatives under the Mallows model, as shown by \citet{braverman2009sorting}. However, the average position is not meaningful under the selective Mallows model, because the lengths of the selective ranking may vary. 

Also, although under the Mallows model, the probability of displacement of an alternative decays exponentially in the length of the displacement, under the selective Mallows model, distant elements might be easily swapped in a sample containing only a small number of the alternatives that are ranked between them in the central ranking. For example, denoting with $\id$ the identity permutation ($1\succ 2\succ \dots \succ n$), if $m\in[n]$, $m \gg 1$, $\sample_1\sim\M_{\id,\b}^{\{[n]\}}$ and $\sample_2\sim\M_{\id,\b}^{\{\{1,m\}\}}$, then:
\[
    \pr[m\succ_{\sample_1} 1] \le 2e^{-\b m/2},
\]
since in order for $1$ and $m$ to be swapped, either has to be displaced at least $m/2$ places, and as shown by \citet{bhatnagar2015lengths}, under Mallows model, the probability of displacement of an alternative by $t$ places is bounded by $2e^{-\b t}$. However:
\[
    \pr[m\succ_{\sample_2} 1] \ge e^{-\b}/(1+e^{-\b}),
\]
using the bound for swap probability provided by \citet{chierichetti2014reconstructing}. 

Since $m \gg 1$: $\pr[m\succ_{\sample_1} 1] \ll \pr[m\succ_{\sample_2} 1]$.

However, even though some selection sets may weaken the concentration property of the positions of the alternatives, we show that the positional estimator works under the selective Mallows model. This happens due to the requirement that each pair of alternatives should appear frequently: the majority of distant (in $\central$) alternatives remain distant in the majority of the resulting incomplete rankings obtained by restricting $\central$ to the selection sets. This is summarized by the following:

\begin{theorem}\label{thm:approx}
Let $\prof\sim\MadvS$, where $\central\in\S_n$, $\b>0$, $|\Scal|=r$ and $\Scal$ is $p$-frequent, for some $p\in(0,1]$, and $\ep\in(0,1)$. Then, for the positional estimator $\positional=\positional[\prof]$, there exists some $N=O(\frac{\b^2+1}{\b^3 p^2r}\log(n/\ep))$ such that:
\[
    \pr[\exists i\in[n]:|\positional(i)-\central(i)|>N]\le \ep\,.
\]
\end{theorem}

\begin{sproof}
We show that with high probability, for any alternative $i\in[n],$ only $N = O(\frac{1}{p}(\frac{1}{\b}+\frac{1}{\b pr}\log\frac{n}{\ep}))$ other alternatives $j$ are ranked incorrectly relatively to $i$ in the majority of the samples of $\prof$ where both $i$ and $j$ appear. Therefore, by the definition of the {\posest}, even after tie braking, each alternative is ranked by the output ranking within an $O(N)$ margin from its original position in $\central$.

If two alternatives $i,j$ are ranked $L$ positions away by the reduced central ranking corresponding to a sample, then the probability that they appear swapped in the sample is at most $2e^{-\b L/2}$. However, even if $i,j$ are distant in $\central$, they might be ranked close by a reduced central ranking.

For any $i\in[n]$, we define a neighborhood $\N_i(L,\lambda)$ containing the other alternatives $j$ which appear less than $L$ positions away from $i$ in the corresponding reduced central rankings of at least $r/\lambda$ samples. Intuitively, those alternatives $j$ outside $\N_i(L,\lambda)$ are far from $i$ in the corresponding reduced central ranking of many samples. Hence, in these samples where $i,j$ are initially far (according to $L$), the probability of observing them swapped is small enough so that, with high probability, the number of samples where they are ranked correctly is dominant among all the appearances of the pair, since we have additionally forced the number of samples where $i,j$ are initially close (in which swaps are easy) to be small (according to $\lambda$).

Additionally, we bound the size of the neighborhood by $|\N_i(L,\lambda)|\le 2L\lambda$, because in each sample there is only a small number of candidate neighbors (according to $L$) and an element of $\N_i(L,\lambda)$ uses many of the total number places. We conclude the proof by setting $L=O(\frac{1}{\b}+\frac{1}{\b pr}\log\frac{n}{\ep}))$ and $\lambda = O(\frac{1}{p})$. Intuitively, $\lambda$ is chosen so that the number of samples where swaps are difficult is comparable to $pr$ (the minimum number of samples where each pair appears). The margin of error is $N=O(L\lambda)$.
\end{sproof}

For the details, we refer the reader to the Appendix \ref{s:pf3}. We continue with a remark on the sample complexity.
\begin{remark}\label{rem:approx}
In Theorem \ref{thm:approx}, the margin of the approximation accuracy $N$ can be refined as follows:
\[
    N = 
    \begin{cases}
        O(\frac{1}{\b p^2r}\log(n/\ep)),\text{ when }r=O(\frac{1}{p}\log(n/\ep))\,. \\
        
        O(\frac{1}{\b p}),\text{ when }r=\omega(\frac{1}{ p}\log(n/\ep))\,. \\
        0, \text{ when }r=\omega(\frac{1}{\b^2 p}\log(n/\ep))\,.
    \end{cases}
\]
\end{remark}

According to Remark \ref{rem:approx}, the error margin of approximation for the {\posest} provably diminishes inversely proportionally to $\b p^2r$, when $r$ is sufficiently small and eventually becomes zero when $r$ exceeds the sample complexity of Theorem~\ref{thm:upper_central}.


\section{Applications of Approximation}
\label{s:applications}

Assume we are given a sample from $\MadvS$, where $\Scal$ is $p$-frequent for some $p\in(0,1]$. Unless $|\Scal|$ is sufficiently large, we cannot find the central ranking with high probability. However, due to Theorem $\ref{thm:approx}$, we know that the positional estimator approximates the true positions of the alternatives within a small margin. This implies two possibilities which will be analyzed shortly: First, in Section~\ref{sec:mle}, we present an algorithm that finds the maximum likelihood estimation of the central ranking with high probability. The algorithm is quite efficient, assuming that the frequency $p$ and the spread parameter $\b$ are not too small. In Section~\ref{sec:topk}, we show how to retrieve the top-$k$ ranking (assuming we have enough samples to sort $O(k)$ elements), when $k$ is sufficiently large.


\subsection{Maximum Likelihood Estimation of the Central Ranking} \label{sec:mle}

We work in the regime where $r$ is (typically much) smaller than the sample complexity of Theorem~\ref{thm:upper_central}. 

We start with some necessary notation. For any $\A\subse \S_n$, let $\mleA=\mleA[\prof]$ be a \textit{maximal likelihood} estimation of $\central$ among elements of $\A$, that is:
\begin{equation}\label{eqn:mleA}
    \mleA\in\arg\max_{\sample\in\A}\pr[\prof|\sample,\b,\Scal]\,.
\end{equation}
If $\A=\S_n$, $\mleA=\mle$, is a \textit{maximum likelihood} estimation of $\central$, while if $\central\in\A$, $\mleA=\mlenat$ is a \textit{likelier than nature} estimation of $\central$ \citep{rubinstein2017sorting}.

The following lemma states that computing the maximum likelihood ranking (\textsc{MLR}) is equivalent to maximizing the total number of pairwise agreements (\textsc{MPA}) between the selected ranking and the samples.

\begin{lemma} \label{lem:mle_form} 
Let $\A$ be a subset of $\S_n$ and assume that we draw a sample profile $\prof \sim \MadvS.$ Consider the following two problems: 
\begin{equation} \label{eq:mlr}
    \textsc{(MLR)} \hspace{2mm} \arg\max_{\sample\in\A}\pr[\prof|\sample,\b,\Scal]\ \ \ \ \ \mbox{and}
\end{equation}
\begin{equation} \label{eq:mpa}
    \textsc{(MPA)} \hspace{2mm}  \arg\max_{\sample\in\A}\sum_{i\succ_\sample j}|\{\sample'\in\prof: i\succ_{\sample'}j\}|\,.
\end{equation}
Then, the solutions of \textsc{(MLR)} and \textsc{(MPA)} coincide.
\end{lemma}

\begin{proof}

If $\prof=\samples \sim \MadvS$, then, starting from \eqref{eq:mlr}, we get that:
\[
    \arg\max_{\sample\in\A}\pr[\prof|\sample,\b,\Scal] = \arg\min_{\sample\in\A}\sum_{\ell\in[r]}\kt(\sample,\sample_\ell)\,.
\]
That is, maximizing the likelihood function is equivalent to minimizing the total number of pairwise disagreements. Equivalently, we have to maximize the total number of pairwise agreements and, hence, retrieving \eqref{eq:mpa}. Note that the samples in $\prof$ are incomplete and therefore each pair of alternatives is compared only in some of the samples.
\end{proof}

Let us consider a subset $\A$ of $\S_n$. According to Lemma \ref{lem:mle_form}, there exists a function $f=f[\prof]: [n]\times [n]\to \Nn$ such that solving (\textsc{MLR}) is equivalent to maximizing a score function $s:\A \to \Nn$ of the form:
\begin{equation}\label{eq:score}
    s(\sample) = \sum_{i\succ_\sample j}f(i,j)\,.
\end{equation}
Then, as shown by \citet{braverman2009sorting}, there exists a dynamic programming, which given an initial approximation of the maximizer of $s$, computes a ranking that maximizes $s$ in time linear in $n$, but exponential in the error of the initial approximation. More specifically, \citet{braverman2009sorting} show that:

\begin{lemma}[\citet{braverman2009sorting}] \label{lem:maxscore}
    Consider a function $f:[n]\times [n] \rightarrow \Nn$. Suppose that there exists an optimal ordering $\sample\in\S_n$ that maximizes the score \eqref{eq:score} such that $|\sample(i)-i|\le R, \forall i\in[n]$. Then, there exists an algorithm which computes $\sample$ in time $O(n\cdot R^2\cdot 2^{6R})$.
\end{lemma}

Recall that the positional estimator finds such an approximation $\positional$ of the central ranking. Also, a careful examination of the proof of Lemma \ref{lem:maxscore} shows that given any initial permutation $\sigma \in \S_n$, the dynamic programming algorithm finds, in time $O(n\cdot R^2\cdot 2^{6R})$, a maximizer (of the score function $s$) in $\A\subse \S_n$, where $\A$ contains all the permutations that are $R$-pointwise close\footnote{We say that $\pi,\sigma \in \S_n$ are $R$-pointwise close, if it holds that: $|\pi(i) - \sigma(i)| \leq R$ for all $i \in [n]$.} to the initial permutation $\sigma$. Therefore, we immediately get an algorithm that computes a likelier than nature estimation $\mlenat$ by finding $\mleA$, for $\A$ such that $\positional,\central\in\A$.

Furthermore, if $\mle$ is an approximation of $\central$, then $\positional$ is an approximation of $\mle$. Hence, we get an algorithm for computing a maximum likelihood estimation $\mle$. It turns out that $\mle$ approximates $\central$, but with a larger margin of error. Thus, we obtain the following:

\begin{theorem}\label{thm:mle_alg}
Let $\prof$ be a sample profile from $\MadvS$, where $\Scal$ is a $p$-frequent selection sequence, $p\in(0,1]$, $|\Scal|=r$, $\central\in\S_n$, $\b>0$ and let $\a>0$. Then:
\begin{enumerate}
    \item There exists an algorithm that finds a likelier than nature estimation of $\central$ with input $\prof$ with probability at least $1-n^{-\a}$ and in time:
    \[
        T = O(n^2+n^{1+O(\frac{2+\a}{r\b p^2})}2^{O(\frac{1}{p\b})}\log^2n)\,.
    \]
    \item There exists an algorithm that finds a maximum likelihood estimation of $\central$ with input $\prof$ with probability at least $1-n^{-\a}$ and in time:
    \[
        T = O(n^2+n^{1+O(\frac{2+\a}{r\b p^4})}2^{O(\frac{1}{p^3\b})}\log^2n)\,.
    \]
\end{enumerate}
\end{theorem}

To summarize the algorithm of Theorem~\ref{thm:mle_alg}, we note that it consists of two main parts. First, using the fact that our samples are drawn from a selective Mallows distribution, in which the positions of the alternatives exhibit some notion of locality, we approximate the central ranking within some error margin for the positions of alternatives. Second, beginning from the approximation we obtained at the previous step, we explore (using dynamic programming instead of exhaustive search, see Lemma \ref{lem:maxscore}) a subset of $\S_n$ which is reasonably small and provably contains with high probability either $\central$ (for finding a likelier than nature ranking) or $\mle$ (for finding a maximum likelihood one).

The proof of Theorem \ref{thm:mle_alg} can be found at Appendix \ref{s:pf4}.


\subsection{Retrieving the Top-$k$ Ranking} \label{sec:topk}

In this section, we deal with the problem of retrieving the top-$k$ sequence and their ranking in $\pi_0$. In this problem, we are given a sample profile from a selective Mallows model and a positive integer $k$. We aim to compute the (identities and the) order of the top-$k$ sequence in the central ranking $\pi_0$.

Based on the properties of the positional estimator, it suffices to show that (after sufficiently many selective samples) any of the alternatives of the top-$k$ sequence $i$ in $\pi_0$ is ranked correctly with respect to any other alternative $j$ in the majority of samples where both $i$ and $j$ appear. Then, every other alternative will be ranked after the top-$k$ by the {\posest}. 

The claim above follows from the approximation property of the {\posest}. Theorem~\ref{thm:approx} ensures that for any alternative of the top-$k$ sequence $i$, the correct majority order against most other alternatives (those that are far from $i$ in most reduced central rankings). So, we can focus on only $O(k)$ pairs, which could appear swapped with noticeable probability. 

To formalize the intuition above, we can restrict our attention to the regime where the number of alternatives $n$ is sufficiently large and $k = \omega(1 / (p \b))$. By Remark~\ref{rem:approx}, this ensures that the accuracy of the approximation of {\posest}  diminishes inversely proportional to $\b p^2r$, since we only aim to ensure that the accuracy is $O(k)$. Specifically, Theorem~\ref{thm:topk-up} provides an upper bound on the estimation in that regime and Corollary~\ref{cor:topk-low} gives a general lower bound for the case where $k = O(n)$.

\begin{theorem}
\label{thm:topk-up}
Let $k = \omega( 1/(p \b))$ be a positive integer. For any $\ep\in(0,1)$ and $p\in(0,1]$, there exists an algorithm which given a sample profile from $\MadvS$, where $\Scal$ is a $p$-frequent selection sequence with:
\begin{equation*}
    |\Scal| = O\!\left(\frac{\log(k/\ep)}{p(1-e^{-\b})^2}+\frac{\log(n/\ep)}{p^2\b k}\right)\,,
\end{equation*}
retrieves the top-$k$ ranking of the alternatives of $\central$, with probability at least $1-\ep$.
\end{theorem}

\begin{sproof}
Let $\prof\sim\MadvS$ be our sample profile. We will make use of the {\posest}  $\positional = \positional[\prof]$ and, without loss of generality, assume that $\central$ is the identity permutation. We will bound the probability that there exists some $i\in[k]$ such that $\positional(i)\neq i$.

For any $i\in[k]$, we can partition the remaining alternatives into $A_1(i) = \N_i(L,\lambda)$ and $A_2(i)=[n]\setminus (A_1(i)\cup\{i\})$. From the proof sketch of Theorem~\ref{thm:approx}, we recall that $\N_i(L,\lambda)$ contains the alternatives that are ranked no more than $L$ places away from $i$ in the reduced central rankings corresponding to at least $r/\lambda$ samples.

From an intermediate result occurring during the proof of Theorem \ref{thm:approx}, it holds that for some $L$, $\lambda$ such that $|A_1(i)|=O(\frac{1}{p\b}+\frac{1}{p^2\b |\Scal|}\log(n/\ep))$, with probability at least $1-\ep/2$, for every $i\in[n]$, for every alternative $j\in A_2(i)$ (distant from $i$ in most samples), $j$ is ranked in the correct order relatively to $i$ in the majority of the samples where they both appear.

Picking $L$, $\lambda$ so that the above result holds, there exists some $r_1=O(\frac{1}{p^2\b k}\log(n/\ep))$ such that, if $|\Scal| \geq r_1$, then $|A_1(i)| = O(k)$. 

Furthermore, following the same technique used to prove Theorem \ref{thm:upper_central}, we get that for some $r_2 = O(\frac{1}{p(1-e^{-b})^2}\log(k/\ep))$, if $|\Scal|\ge r_2$, then, with probability at least $1-\ep/2$, every pair of alternatives $(i,j)$ such that $i\in[k]$ and $j\in A_1(i)$ is ranked correctly by the majority of samples where both $i$ and $j$ appear, since the total number of such pairs is $O(k^2)$.

Therefore, with probability at least $1-\ep$, both events hold and for any fixed $i\in[k]$, $\positional(i)=i$, because $i$ is ranked correctly relatively to every other alternative in the majority of their pairwise appearances and also because for every other alternative $j>k$: $\positional(j)>k$, since each of the alternatives in $[k]$ appear before it in the majority of samples where they both appear.
\end{sproof}

For the complete proof, we refer to Appendix \ref{s:pf5}.

From a macroscopic and simplistic perspective, the sample complexity of learning the top-$k$ ranking can be interpreted as follows. The first term, i.e., $O_\b(\frac{1}{p}\log(k/\ep))$, accounts for learning the ranking of the top-$k$ sequence (as well as some $O(k)$ other alternatives), since each of them is close to the others in the central ranking (and in each reduced rankings where they appear). Hence, it is probable that their pairs appear swapped. The second term, i.e., $O_\b(\frac{1}{p^2 k}\log(n/\ep))$, accounts for identifying the top-$k$ sequence, by approximating their positions. Intuitively, the required accuracy of the approximation diminishes to the value of $k$, since we aim to keep $O(k)$ probable swaps for each of the alternatives of the top-$k$ sequence. Combining the two parts, we conclude that given enough samples, {\posest} outputs a ranking where the top-$k$ ranking coincides with the top-$k$ ranking of $\pi_0$.

We conclude with the lower bound, followed by a discussion about the tightness of our results.

\begin{cor} \label{cor:topk-low}
For any $k\le n$, $p\in(0,1]$, $\ep\in(0,1/2]$, $\b>0$ and $r=o(\frac{1}{\b p}\log(k/\ep))$, there exists a $p$-frequent selection sequence $\Scal$ with $|\Scal|=r$, such that for any central ranking estimator, there exists a $\central\in\S_n$ such that the estimator, given a sample profile from $\MadvS$, fails to retrieve the top-$k$ ranking of $\central$ with probability at least $\ep$.
\end{cor}

Corollary~\ref{cor:topk-low} is an immediate consequence of Theorem~\ref{thm:lower_central}. The bounds we provided in Theorem~\ref{thm:topk-up}~and Corollary~\ref{cor:topk-low} become essentially tight if $k=\Omega(\frac{1}{p}\log(n/\ep))$, since the term $O_\b(\frac{1}{p}\log(k/\ep))$ becomes dominant in the upper bound, which then essentially coincides with the lower bound. In the intuitive interpretation we provided for the two terms of the sample complexity in Theorem~\ref{thm:topk-up}, this observation suggests that when $k$ is sufficiently large, the sample complexity of identifying the top-$k$ ranking under the selective Mallows model is dominated by the sample complexity of sorting them.

We conclude with an informative example, where we compare the sample complexity of retrieving the complete central ranking and the top-$k$ ranking in an interesting special case. Let us assume that $p=1/\log\log n$ and that $k=\Theta_\b(\log(n/\ep))$. Then we only need $O_{\b}(\log\log n \cdot \log\frac{\log n}{\ep})$ samples to retrieve the top-$k$ ranking, while learning the complete central ranking requires $\Omega_{\b}(\log (n/\ep))$ samples. Namely, we have an almost exponential improvement in the sample complexity, for values of $k$ that suffice for most practical applications.

\section{Experiments}
\label{s:experiments}

In this section, we present some experimental evaluation of our main results, using synthetic data. First, we empirically verify that the sample complexity of learning the central ranking from $p$-frequent selective Mallows samples using $\posest$ is $\Theta(1/p)$, assuming every other parameter to be fixed. Furthermore, we illustrate empirically that $\posest$ is a smooth estimator of the central ranking, in the sense that $\posest$ outputs rankings that are, on average, closer in Kendall Tau distance to the central ranking as the size of the sample profile grows.

\subsection{Empirical sample complexity}

We estimate the sample complexity of retrieving the central ranking from selective Mallows samples where $n=20$ and $\beta=2$, with probability at least $0.95$, using $\posest$ by performing binary search over the size of the sample profile. During a binary search, for every value, say $r$, of the sample profile size we examine, we estimate the probability that $\posest$ outputs the central ranking by drawing $100$ independent $p$-frequent selective Mallows profiles of size $r$, computing $\posest$ for each one of them and counting successes. We then compare the empirical success rate to $0.95$ and proceed with our binary search accordingly. For a specific value of $p$, we estimate the corresponding sample complexity, by performing $100$ independent binary searches and computing the average value. The results, which are shown in Figure \ref{fig:SCadv}, indicate that the dependence of sample complexity on the frequency parameter $p$ is indeed $\Theta(1/p)$.
\begin{figure}[ht]
    \centering
    \includegraphics[width=.85\linewidth]{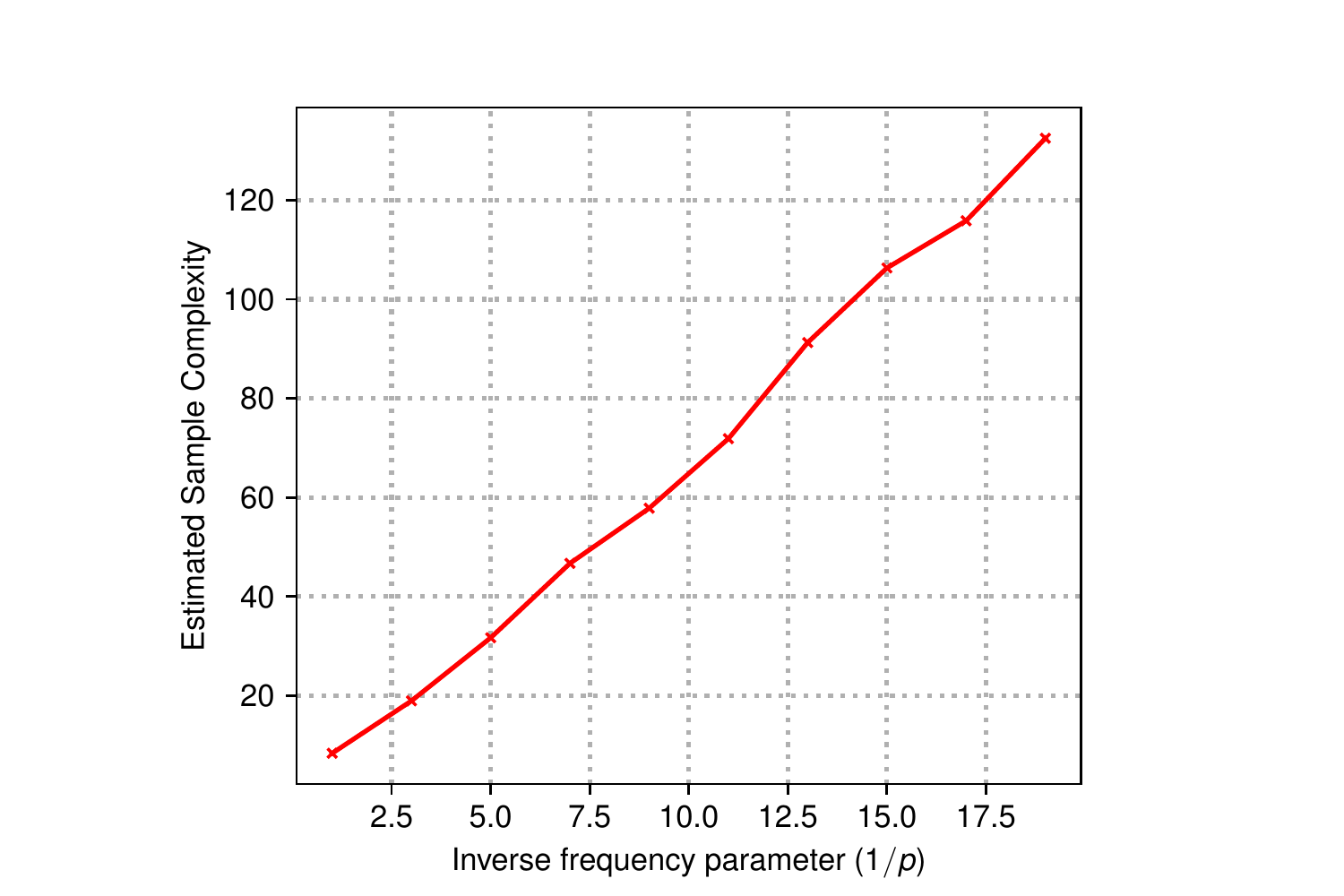}
    \caption{Estimated sample complexity of retrieving, with probability at least $0.95$ and using $\posest$, the central ranking from selective Mallows samples, with $n=20$, $\beta=2$, over the frequency parameter's inverse.}\label{fig:SCadv}
\end{figure}

\subsection{Smoothness of $\posest$}
We plot, for different values of the frequency parameter $p$, the average Kendall Tau distance between the central ranking and the output of $\posest$ with respect to the size of the sample profile. For each value, say $r$, of the sample profile size, considering $\beta=0.3$ and $n=20$, we draw $100$ independent selective Mallows sample profiles, each of size $r$, we compute the distance between the output of $\posest$ for each sample profile and the central ranking and take the average of these distances. The results are presented in Figure \ref{fig:PDadv}.

\begin{figure}[ht]
    \centering
    \includegraphics[width=.95\linewidth]{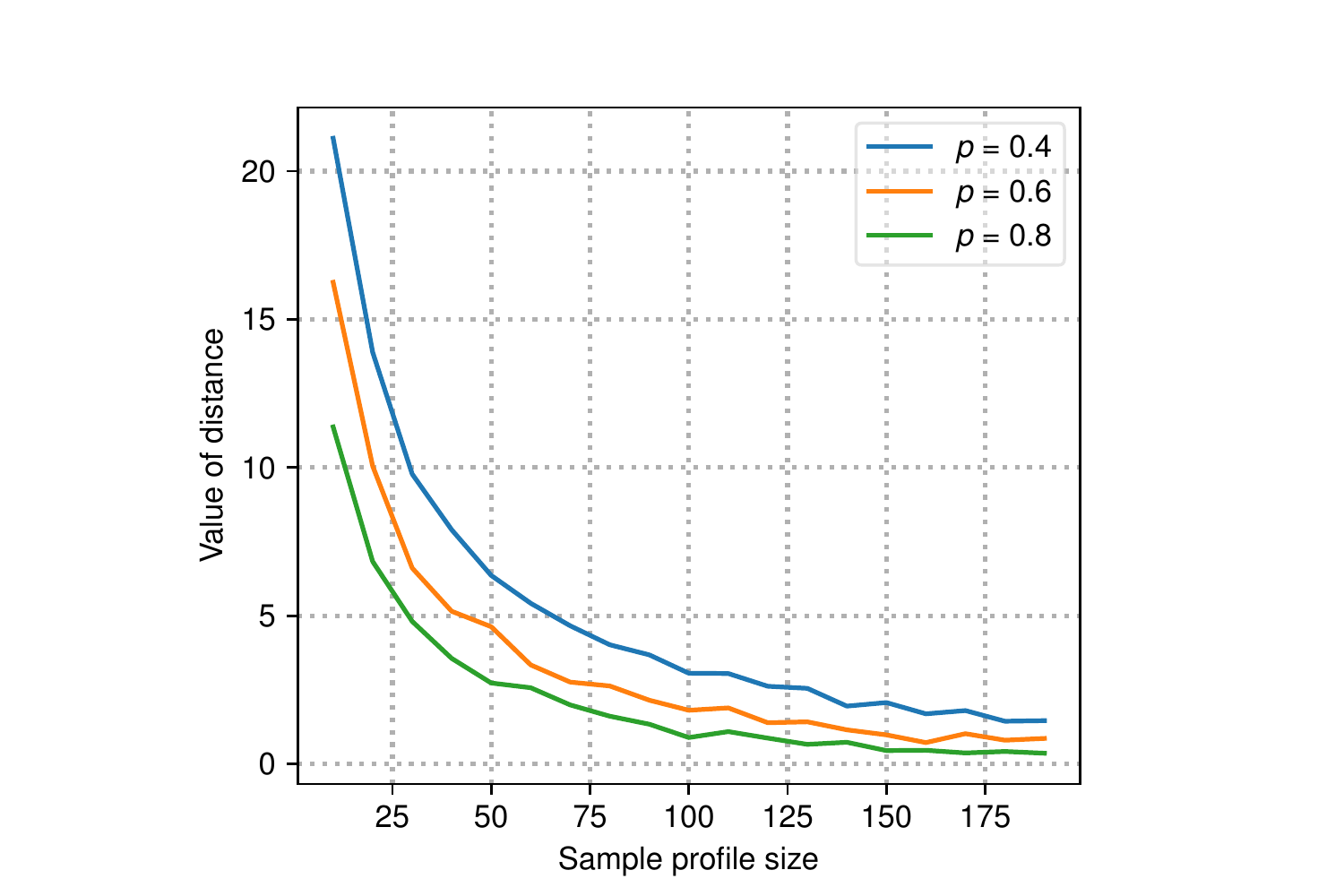}
    \caption{Average Kendall Tau distance between the output of $\posest$ and the central ranking with respect to the size of the sample profile, for different values of the frequency parameter $p$, when $n=20$, $\beta = 0.3$.}\label{fig:PDadv}
\end{figure}

\subsubsection*{Acknowledgements}
The authors would like to thank the anonymous reviewers for their valuable comments and suggestions.

Dimitris Fotakis and Alkis Kalavasis are partially supported by the Hellenic Foundation for Research and Innovation (H.F.R.I.) under the ``First Call for H.F.R.I. Research Projects to support Faculty members and Researchers and the procurement of high-cost research equipment grant'',  project BALSAM, HFRI-FM17-1424.

\newpage

\bibliography{refs}

\appendix
\appendixpage

\section{Proof of Theorem \ref{thm:upper_central}}\label{s:pf1}
\label{appendix:start}
The proof of Theorem \ref{thm:upper_central} follows the same steps as the proof provided by \citet{caragiannis2013noisy} for the upper bound of the sample complexity of finding the central ranking given complete Mallows samples.

Assume that $\Scal$ is $p$-frequent and $\prof\sim\MadvS$. Also, without loss of generality, let $\central$ be the identity permutation $\id$. For the following, we denote with $\numle[i,j]=\numle[i,j][\prof]$ the number of samples where $i\succ j$, for any $i,j\in[n]$ and with $\num[i,j]=\num[i,j](\Scal)$ the total number of samples where they both appear. We, then, have that:
\[
    \pr[\positional[\prof]\neq \central] \le \pr[\exists i<j: \numle[i,j]<\numle[j,i] ]\le \sum_{i<j}\pr[\numle[i,j]<\numle[j,i]]\,.
\]
For any fixed pair $i<j$, the probability that $\numle[i,j]<\numle[j,i]$ is maximized in the case when they are adjacent in $\central$. In this case, $\numle[i,j]=\sum_{\ell\in\num[i,j]}X_\ell$, where $X_\ell\sim\be(e^{-\b}/(1+e^{-\b})),\forall \ell\in[\num[i,j]]$. Let $Y_\ell=1-X_\ell,\forall \ell\in[\num[i,j]]$. Then, we have:
\[
    \pr[\numle[i,j]<\numle[j,i]]\le \pr[\sum_{\ell\in[\num[i,j]]}X_\ell-Y_\ell\ge 0] \le \exp(-2\num[i,j](\frac{1-e^{-\b}}{1+e^{-\b}})^2) \le \exp(-2p{\r}(\frac{1-e^{-\b}}{1+e^{-\b}})^2),
\]
where the second inequality follows from the Hoeffding bound and the third from the fact that $\Scal$ is $p$-frequent.

Demanding that the last term is less than $\ep$ and solving for ${\r}$ concludes the proof of Theorem \ref{thm:upper_central}. 

\section{Proof of Theorem \ref{thm:lower_central}}
\label{s:pf2}

Let $\Scal$ contain $p{\r}$ full sets and $(1-p){\r}$ sets of size at most $n\sqrt{p/(1-p)}$. For any $i,j\in[n]$, let $\num[i,j](\Scal)$ be the number of sets of $\Scal$ containing both $i$ and $j$, that is, the number of the appearances of pair $(i,j)$.

Clearly, $\Scal$ is $p$-frequent and: 
\begin{equation}\label{eqap:bounded_appearances}
    \sum_{i<j}\num[i,j](\Scal)\le pn^2{\r},
\end{equation} 
since each full set has no more than $n^2/2$ pairs of alternatives and each of the remaining sets have no more than $n^2p/(2(1-p))$ pairs.

Assume that ${\r}<\frac{1}{8p\b}\log(\frac{n(1-\ep)}{4\ep})$. From Eq.~\eqref{eqap:bounded_appearances} we get that:
\begin{equation}\label{eqap:bounded_appearances2}
    \sum_{i<j}\num[i,j](\Scal)< \frac{n^2}{8\b}\log\left(\frac{n(1-\ep)}{4\ep}\right)\,.
\end{equation} 

We will show that there exists a set of $n/2$ disjoint pairs of alternatives which we observe only a few times in the samples. For simplicity, assume that $n/2\in\Nn$. Consider the following family $\{P_t\}_{t\in[n/2]}$ of perfect matchings on the set of alternatives, that is, $n/2$ sets of $n/2$ disjoint pairs:
\begin{align*}
    P_1 = \{(1,2),(3,&4),\dots,(n-1,n)\}\,, \\
    P_2 = \{(1,4),(3,&6),\dots,(n-1,2)\}\,, \\
    & \dots \\
    P_t = \{(1,(2t)\mod n),\dots, (&n-1,(2t+n-2)\mod n)\}\,, \\
    & \dots \\
    P_{n/2} = \{(1,n),(3,2&),\dots,(n-1,n-2)\}\,.
\end{align*}

Observe that no pair of alternatives appears in more than one perfect matching of the above family. Therefore:

\begin{equation}\label{eqap:bounded_appearances_matchings}
    \sum_{t\in[n/2]}\sum_{(i,j)\in P_t}\num[i,j](\Scal)\le\sum_{i<j}\num[i,j](\Scal)\,.
\end{equation}

Therefore, combining Eq.~\eqref{eqap:bounded_appearances2} and Eq.~\eqref{eqap:bounded_appearances_matchings}, we get that:

\begin{equation*}\label{eqap:bounded_appearances_matching}
    \exists t\in[n/2]: \sum_{(i,j)\in P_t}\num[i,j](\Scal) < \frac{n}{4\b}\log\left(\frac{n(1-\ep)}{4\ep}\right)\,.
\end{equation*}

Therefore, since $|P_t|=[n/2]$, there exist at least $n/4$ pairs $(i,j)\in P_t$ with:
\begin{equation}\label{eqap:bounded_appearances_pairs}
    \num[i,j](\Scal) < \frac{1}{\b}\log\left(\frac{n(1-\ep)}{4\ep}\right)\,.
\end{equation}

We proceed with an information-theoretic argument which is based on the observation that if the pairs of $P=P_t$, for which Eq.~\eqref{eqap:bounded_appearances_pairs} holds, are adjacent in the central ranking, then the probability of swap is maximized for each pair and also knowledge of the pairwise orders of any fraction of the pairs does not give information about the pairwise order of any remaining pair. For simplicity and without loss of generality, assume that $P=P_1$, 

For the selection vector $\Scal = (S_1,\dots, S_{|\Scal|})$ we denote with $\sym(\Scal)$ the support of $\MadvS$, namely the set of vectors containing in any position $\ell\in[|\Scal|]$ a permutation of $\Set_\ell$. For the following, $\fprof$ will be used to denote elements of $\sym(\Scal)$, while $\prof\sim\MadvS$ will be a random variable.

Let $\estim$ be any (randomized) algorithm for estimating the central ranking. With the notation $\pr[A]$ we refer to the probability of the event $A$, taking into consideration any randomness involved in $A$ (for example the randomness of $\prof$ and $\estim$). Also, let $\F\subse \S_n$ be as follows:

\begin{equation}\label{eqap:perm_family}
    \F = \{\sample\in\S_n: \{\sample(2i-1),\sample(2i)\}=\{2i-1,2i\}\}\,.
\end{equation}

For example, if $n=4$ then $\F = \{1\succ 2\succ 3\succ 4, 1\succ 2\succ 4\succ 3, 2\succ 1\succ 3\succ 4, 2\succ 1\succ 4\succ 3\}$.

Fix $\sample\in\F$. For any $\sample'\in\F$, let $D(\sample')$ denote the number of pairwise disagreements between $\sample$ and $\sample'$ on the elements of $P$:
\[
    D(\sample') = \{(i,j)\in P: (\sample(i)-\sample(j))(\sample'(i)-\sample'(j))<0\}\,.
\]

Fix $\fprof\in\sym(\Scal)$. Then, assuming the following notation: 
\[
    \pr[\fprof|\sample] := \pr_{\prof\sim\M_{\sample,\b}^\Scal}[\prof = \fprof],
\]
we apply the triangle inequality property of Kendall tau distance and get that:

\begin{equation}\label{eqap:obs_prob_bound}
    \pr[\fprof|\sample]\ge e^{-\b\sum_{(i,j)\in D(\sample')}\num[i,j](\Scal)}\pr[\fprof|\sample']\,.
\end{equation}

Also, since the estimator $\estim$ must have a single output:

\begin{equation}\label{eqap:single_output}
    \sum_{\sample'\in\F}\pr[\estim[\fprof]=\sample']\le 1\,. 
\end{equation}

We multiply the terms of Ineq.~\eqref{eqap:single_output} with $\pr[\fprof|\sample]$, apply Ineq.~\eqref{eqap:obs_prob_bound} and sum over all $\fprof\in\sym(\Scal)$ to get:

\begin{align*}
    \sum_{\sample'\in\F}&e^{-\b\sum_{(i,j)\in D(\sample')}\num[i,j](\Scal)} \sum_{\fprof\in\sym(\Scal)}\pr[\estim[\fprof]=\sample']\pr[\fprof|\sample']\le \sum_{\fprof\in\F}\pr[\fprof|\sample]\le 1\,.
\end{align*}

Assume, for contradiction, that for every $\sample'\in\F$ it holds:
\begin{align*}
    \pr_{\prof\sim\M_{\sample,\b}^\Scal}[\estim[\prof]=\sample']
    \ge 1-\ep\,. 
\end{align*}

Then, since it holds that $\pr_{\prof\sim\M_{\sample,\b}^\Scal}[\estim[\prof]=\sample'] = \sum_{\fprof\in\sym(\Scal)}\pr[\estim[\fprof]=\sample']\pr[\fprof|\sample']$, we get:

\begin{align*}
    (1-\ep)\sum_{\sample'\in\F}&e^{-\b\sum_{(i,j)\in D(\sample')}\num[i,j](\Scal)} \le 1\,.
\end{align*}

However, from Ineq.~\eqref{eqap:bounded_appearances_pairs}, we get that:
\[
    \sum_{\sample'\in\F} e^{-\b\sum_{(i,j)\in D(\sample')}\num[i,j](\Scal)} > 1+\frac{n}{4}\frac{4\ep}{n(1-\ep)}\,.
\]

We conclude that: $1-\ep+\ep<1,$ contradiction.

\section{Proof of Theorem \ref{thm:approx}}\label{s:pf3}

For the following, for any ranking $\pi\in\S_n$ and $S\subse [n]$, let $\pi|_S$ denote the reduced central ranking of $\pi$ to $S$, that is, the permutation\footnote{Specifically, for a permutation $\pi : [n] \rightarrow [n]$ and $S \subseteq [n]$, the mapping $\pi|_S$ is a bijection from $S$ to $[|S|].$} of the elements of $S$ that agrees with their order in $\pi$.

Assume that $\prof\sim\MadvS$, where $\Scal$ is $p$-frequent. The proof of Theorem \ref{thm:approx} is based on the definition of a notion of neighborhood for each alternative $i\in[n]$. In particular, for every $i\in[n]$, $L>0$ and $\lambda>0$ we define $\N_i(L,\lambda)=\N_i(L,\lambda)[\prof]$ to be the subset of $[n]$ containing all the alternatives $j$ for which there exist at least $r/\lambda$ sets $S$ of $\Scal$ for each of which it holds that $|\central|_S(i)-\central|_S(j)|\le L$.

Observe that the neighborhoods are formed according to the input data but they are unknown to the algorithm, since $\central$ is unknown.

Furthermore, the following Lemma holds, that controls the neighborhoods' size.

\begin{lemma}\label{lem:neigh_len}
For every $i\in[n]$, $L>0$, $\lambda>0$, it holds that:
\[
    |\N_i(L,\lambda)|\le 2\lambda L\,.
\]
\end{lemma}
\begin{proof}
There are at most $2{\r}L$ total positions for the neighbors of $i$ and each neighbor takes at least ${\r}/\lambda$ of them.
\end{proof}

We now prove that we can pick $L$ and $\lambda$ so that, with high probability, for any $i\in[n]$, every element outside its neighborhood $\N_i(L,\lambda)$ is ranked correctly relatively to $i$ in the majority of samples where they both appear. More specifically, the following Lemma holds:

\begin{lemma}\label{lem:correct_majority_orders}
Assume that $\prof=\samples\sim\MadvS$, $c\in(0,1/2]$, $c$ constant and $\ep\in(0,1)$. Then, there exists some constant $C=C(c)$ such that, by considering:
\[
    L = \frac{C}{\b}+\frac{2}{\b c p r}\log(2n^2/\ep) \text{ and } \lambda = \frac{2}{(1-c)p},
\]
it holds that:
\[
    \numle[i,j]\ge (1-c)\num[i,j]\,,
\]
for every $i\in[n]$ and for every $j\in[n]\setminus (\N_i(L,\lambda)\cup \{i\})$ with probability at least $1-\ep$.
\end{lemma}

\begin{proof}
Fix $i\in[n]$ and $j\in [n]\setminus (\N_i(L)\cup \{i\})$. It is sufficient to show that the probability of the event that $\numle[j,i]>c\num[i,j]$ is less than $\ep/n^2$, since, in that case, the union of the corresponding events over all pairs $i',j'$ such that $j'\in [n]\setminus (\N_{i'}(L)\cup \{i'\})$ would hold with probability less than $\ep$.

We have that ${\r}\ge \num[i,j]\ge p{\r}$. From the selection of $j$ we have that the number of elements of the set $\I$ of indices $\ell\in[{\r}]$ such that: $|\central|_{S_\ell}(i)-\central|_{S_\ell}(j)|>L$ ($i$ and $j$ are initially distant and hence they do not appear swapped in the corresponding sample, with high probability) is at least $\num[i,j]-(1/\lambda){\r}$.

For some index $\lambda\in\I$ we have that the probability that $i$ and $j$ appear swapped in $\sample_\ell$ is upper bounded as follows, according to \citet{bhatnagar2015lengths}:
\[
    \rho = \pr[(\sample_\ell(i)-\sample_\ell(j))(i-j)<0] \le 2e^{-\b L/2}\,.
\]
Assume that $i<j$. We will show that in most of the $\num[i,j]$ samples where both $i$ and $j$ appear, it holds that $i \succ_{\pi_{\ell}} j$. That is, following the notation introduced in Section \ref{s:pf1} of the supplement:
\[
    \pr[\numle[i,j]>c\num[i,j]]<\ep/(2n^2)\,.
\]
Let $\numle[i,j]'$ be the random variable that corresponds to the number of indices $\ell\in\I$ for which $i\succ_{\sample_\ell}j$. Clearly:
\[
    \pr[\numle[i,j]>c\num[i,j]]\le \pr[\numle[i,j]'>c\num[i,j]]\,.
\]
Also, let $\num[i,j]'=|\I|(\ge \num[i,j]-(1/\lambda){\r})$. The random variable $\numle[i,j]'$ follows the Binomial distribution $\bin(\num[i,j]',\rho)$. We want to find some constant $c'=c'(c)\in(0,1)$ such that $c'\num[i,j]'\ge c\num[i,j]$, because, in this case, it suffices to bound the following probability:
\[
    \pr[\numle[i,j]'>c'\num[i,j]']\,.
\]
We pick $c'=c\lambda p/(\lambda p-1)$. For the selected $\lambda$, $c'$ is indeed a constant (since $c$ is considered a constant) taking some value within the interval $(0,1)$ (since $c\in(0,1)$) and from the Chernoff bound we get:
\[
    \pr[\numle[i,j]'>c'\num[i,j]']\le \exp(-\num[i,j]'\kl(c' \parallel \rho))\le \exp(-(p-1/\lambda)r\kl(c' \parallel \rho)),
\]
where $\kl(c' \parallel \rho) = c'\log(c'/\rho)+(1-c')\log((1-c')/(1-\rho))\ge c'\b L/2- C'$, where $C' = C'(c')>0$ is some positive constant.\footnote{Let $\mu, \nu$ be two discrete probability measures on $\Omega$. The Kullback–Leibler divergence between $\mu$ is defined as $\kl(\mu \parallel \nu ) = \sum_{x \in \Omega} \mu(x) \log(\frac{\mu(x)}{\nu(x)})$} For the selected $L$, if $C=2C'/c'$, we get that $\pr[\numle[i,j]'>c'\num[i,j]']\le \ep/(2n^2)$, concluding the proof of Lemma \ref{lem:correct_majority_orders}.
\end{proof}

We conclude the proof of Theorem \ref{thm:approx} by combining Lemmata \ref{lem:neigh_len} and \ref{lem:correct_majority_orders} to get that with probability at least $1-\ep$, for every alternative $i\in[n]$, the number of other alternatives with which $i$ is ordered reversely in the majority of samples where they both appear is upper bounded by $N$:
\[
    N = \frac{2C}{(1-c)p\b}+\frac{4}{\b c(1-c)pr}\log(2n^2/\ep)\,.
\]
Hence, after tie braking, the resulting permutation ranks each alternative no more than $2N$ places away from its position in $\central$. More specifically, if before breaking ties, for $i,j\in[n]$ we have $\positional(j)\le \positional(i)$ then: $\positional(j)+N\le \positional(i)+N$. However: $j \le \positional(j)+N$ and $\positional(i)+N\le i+2N $ therefore: $j\le i+2N\Rightarrow j-i\le 2N$. Therefore, after tie braking: $\positional(i)\le i+2N$. Symmetrically: $\positional(i)\ge i-2N$.

Combining the above results with Theorem \ref{thm:upper_central}, we get Remark \ref{rem:approx} which implies the results presented in Theorem \ref{thm:approx}, where we emphasize that the error of the approximation diminishes to $1/r$.

\section{Proof of Theorem \ref{thm:mle_alg}}\label{s:pf4}

Theorem \ref{thm:mle_alg} consists of two parts. The first one considers the runtime of an algorithm finding a likelier than nature estimation of the central ranking given $p$-frequent selective Mallows samples while the second one refers to solving the maximum likelihood estimation problem. Both parts are based on Lemma \ref{lem:maxscore} (which originates to the work of \citet{braverman2009sorting}).

\paragraph{Part 1.} A careful examination of the proof of Lemma \ref{lem:maxscore} (which can be found in \citep{braverman2009sorting}) reveals that if $\central\in\A=\{\sample\in\S_n:|\sample(i)-\positional(i)|\le N,\forall i\in[n]\}$, then, we can find a likelier than nature estimation $\mlenat$ in time $O(n\cdot N^2 \cdot 2^{6N})$. Picking $N$ according to Theorem \ref{thm:approx}, $\central$ is indeed an element of $\A$ and therefore, we get the desired result.

\paragraph{Part 2.} In this case, we want to show that $\mle\in\A'=\{\sample\in\S_n:|\sample(i)-\positional(i)|\le N',\forall i\in[n]\}$. We claim that with probability at least $1-\ep/2$: $\{|\mle(i)-\central(i)|\le K,\forall i\in[n]\}$ ($\mle$ and $\central$ are pointwise close) for some $K=O(\frac{1}{\b p^3}+\frac{1}{\b p^4 r}\log(n/\ep))$. Therefore, picking $N'=K+N=O(K)$, we have that $\mle\in\A'$, which gives the desired result.

To prove our claim, we generalize the proof that the maximum likelihood estimation of the central ranking from complete Mallows samples is pointwise close to the central ranking. 

Assume, without loss of generality, that $\pi_0= \id$. Let $h>0$ and $c\in(0,1/2),$ which will be defined later. For any $i,j\in[n]$, $\numle[i,j]$ is the number of samples where $i\succ j$ and $\num[i,j]$ is the number of samples where both $i$ and $j$ appear. Clearly, it holds that $\numle[i,j]+\numle[j,i] = \num[i,j]$. 

From Lemmata \ref{lem:correct_majority_orders} and \ref{lem:neigh_len}, with probability at least $1-\ep$, there exist $\lambda$ and $L$ such that for every alternative $i\in[n]$ and any constant $c\in(0,1/2]$ there exists some constant $C=C(c)$ such that :
\begin{enumerate}
    \item\label{enum:f1} $|\N_i(L,\lambda)|\le N = \frac{2C}{\b(1-c) p} + \frac{4}{\b(1-c)c p^2 r}\log(2n^2/\ep)$\,.
    \item\label{enum:f2} $j\in \{i+1,i+2,\dots,n\}\setminus \N_i(L,\lambda) \Rightarrow \numle[j,i] \le c\num[i,j]$ (and $\num[i,j] \ge  p r$).
    \item\label{enum:f3} Symmetrically, for $j\in\{1,\dots,i-1\}\setminus \N_i(L,\lambda)$: $\numle[i,j] \le c\num[i,j]$.
\end{enumerate}

Fix $i\in[n]$ such that $|\mle(i)-i|=K,$ where $K\ge h N$. Without loss of generality, assume $\mle(i)=i+K.$
It suffices to find values of $c$ and $ h$ that contradict the assumption $\mle(i)=i+K$.

Let $\St = \{j\in[n]: i\le\mle(j)<i+K\}$ and: $\St_1 = \{j\in\St: j<i\}$, $\St_2 = \{j\in\St:j\in\N_i(L,\lambda)\}$, $\St_3 = \{j\in\St: j>i \text{ and } j\not\in \N_i(L,\lambda)\}$. Apparently: $\St=\St_1\cup\St_2\cup\St_3$. 

Observe that since $\mle$ maximizes the following score function: 
\begin{align*}
    s: \,\, \S_n & \to \Nn \\
     \pi & \to s(\pi) = \sum_{i_1 \succ_\pi i_2}\numle[i_1,i_2] \,.
\end{align*}
It must hold that:
\[
    0\le \sum_{j\in\St}(\numle[j,i]-\numle[i,j]) \,.
\]

For any $j\in\St_3,$ from item (\ref{enum:f1}) it holds that: 
\[
    \numle[j,i]\le c\num[i,j]\Rightarrow \numle[j,i] - \numle[i,j] \le -(1-2c) p r\,.
\]
Let $|\St_1|=T$. Furthermore: $|\St_2|\le N$ and $|\St_3|\ge K-N-T\ge (  h-1)N-T$. Therefore, we have that:
\[
    0\le rT+rN-(1-2c) p r((h-1)N-T)\Rightarrow
\]
\begin{equation}\label{eqap:ineqT}
    T\ge \frac{(1-2c) p (  h-1)-1}{1+(1-2c) p}N\,.
\end{equation}

Observe that, since there are at least $T$ alternatives $j<i$ such that $\mle(j)\ge i$, say $T_1\subset[n]$, there must be at least $T$ alternatives $j\ge i$ such that $\mle(j)<i$, say $T_2\subset[n]$. Let $H_1=\{1,...,i-1\}$ and $H_2=\{i,...,n\}$. We construct $\sigma_0\in\S_n$ by concatenating $\mle|_{H_1}$ and $\mle|_{H_2}$. It remains to select appropriate values for $  h$ ($h$ does not need to be constant) and $c$ (must be constant) for which $s(\sigma_0)>s(\mle)$, which is a contradiction.

Create the following sets:
\begin{enumerate}
    \item $P_1$: The pairs of elements $i_1,i_2\in[n], i_1<i_2$ for which $i_2\in\N_{i_1}(L,\lambda)$ and $\sigma_0,\mle$ disagree on their relative ranking. Note that: $|P_1|\le 2TN$.
    \item $P_2$: The pairs of elements $i_1,i_2\in[n], i_1<i_2$ for which $\sigma_0,\mle$ disagree, but $i_2\not\in\N_{i_1}(L,\lambda)$ (and $i_1\not\in\N_{i_2}(L,\lambda)$). Note that $\sigma_0$ has the right answer for this pair and $q(i_1\succ i_2)-q(i_2\succ i_1)\ge (1-2c)pr$. Also: $|P_2|\ge T(T-N)$ (select an element of $T_1$ and an element of $T_2$ which is not in the first element's neighborhood).
\end{enumerate}

Then: $s(\sigma_0)-s(\mle)=\sum_{(i_1,i_2)\in P_1}(\numle[i_1,i_2]-\numle[i_2,i_1])+\sum_{(i_1,i_2)\in P_2}(\numle[i_1,i_2]-\numle[i_2,i_1])\ge -2rTN+(1-2c)prT(T-N)=rT((1-2c)pT-((1-2c)p+2)N )$\,.

Using Ineq.~(\ref{eqap:ineqT}), we get that:

\begin{equation*}
    s(\sigma_0)-s(\mle) \ge rTN\left[ \frac{(1-2c)p((1-2c)p(  h-1)-1)}{1+(1-2c)p}-(2+(1-2c)p) \right]\,.
\end{equation*}

We search for values of $c$ and $  h$ such that the quantity inside the brackets is positive. After some algebra, we choose $c<1/4$ (constant) and $  h=2+8/p+8/p^2=O(\frac{1}{p^2})$.

\section{Proof of Theorem \ref{thm:topk-up}}\label{s:pf5}

The proof we provide almost coincides with the sketch we provided in the main part. However, here we have established the appropriate notation that enables us to be more formal.

Let $\prof\sim\MadvS$ be our sample profile. We will make use of the {\posest}  $\positional = \positional[\prof]$ and, without loss of generality, assume that $\central$ is the identity permutation $\id$. We will bound the probability that there exists some $i\in[k]$ such that $\positional(i)\neq i$.

For any $i\in[k]$, we can partition the remaining alternatives into $A_1(i) = \N_i(L,\lambda)$ and $A_2(i)=[n]\setminus (A_1(i)\cup\{i\})$.

From Lemma \ref{lem:correct_majority_orders}, it holds that for some $L$, $\lambda$ such that $|A_1(i)|=O(\frac{1}{p\b}+\frac{1}{p^2\b |\Scal|}\log(n/\ep))$, with probability at least $1-\ep/2$, for every $i\in[n]$, for every alternative $j\in A_2(i)$ it holds: $(\numle[i,j]-\numle[j,i])(j-i)>0$.

Picking $L$, $\lambda$ so that the above result holds, there exists some $r_1=O(\frac{1}{p^2\b k}\log(n/\ep))$ such that, if $|\Scal| \geq r_1$, then $|A_1(i)| = O(k)$. 

Furthermore, following the same technique used to prove Theorem \ref{thm:upper_central}, we get that, if $|\Scal|\ge r_2$ for some $r_2 = O(\frac{1}{p(1-e^{-b})^2}\log(k/\ep))$, then, with probability at least $1-\ep/2$, for every pair of alternatives $(i,j)$ such that $i\in[k]$ and $j\in A_1(i)$ it holds that: $(\numle[i,j]-\numle[j,i])(j-i)>0$, since the total number of such pairs is $O(k^2)$.

Therefore, with probability at least $1-\ep$, both events hold and for any fixed $i\in[k]$, $\positional(i)=i$, because for all $j$: $(\numle[i,j]-\numle[j,i])(j-i)>0$ and also because for every other alternative $j>k$, we have that $\positional(j)>k$, since $\numle[i,j]>\numle[j,i],\forall i\in[k]$.

\section{Randomized $p$-frequent condition}
\label{appendix:rand}

In the main part of the paper, we focus on the class of $p$-frequent queries with $p > 0$. An interesting extension is to deal with random query sets, which may fall in the case where $p = 0$. In fact, we
can consider a randomized setting, where there is a distribution over selection sets and the selection
sequence consists of independent samples from this distribution. The corresponding “$p$-frequent” condition is that
each pair appears with probability at least $p$ in a random selection set. This setting smoothly relaxes the $p$-frequent
setting. We now formally define this relaxed randomized condition.

We consider an alternative assumption to the $p$-frequent one, namely the randomized $p$-frequent assumption and we present empirical data that indicate that at least some of our results under the $p$-frequent assumption continue to hold under the randomized one.

\paragraph{Definition.} We say that a distribution $\D$ supported on $2^{[n]}$ is $p$-frequent if for any $i,j\in[n]$ it holds:
\[
    \pr_{S\sim \D}[S\ni i,j] \ge p\,.
\] 
If the selection sets are independently drawn according to some $p$-frequent distribution, then we say that the selective Mallows model is randomly $p$-frequent. Note that the randomized $p$-frequent assumption, contrary to the simple $p$-frequent assumption, permits the event that some pair of alternatives never appear together in the samples, yet with small probability.

\paragraph{Empirical evaluation.} We estimate the sample complexity of retrieving the central ranking from randomly selective Mallows samples where $n=20$ and $\beta=2$, with probability at least $0.95$, using $\posest$ by performing binary search over the size of the sample profile. During a binary search, for every value, say $r$, of the sample profile size we examine, we estimate the probability that $\posest$ outputs the central ranking by drawing $100$ independent randomly $p$-frequent selective Mallows profiles of size $r$, computing $\posest$ for each one of them and counting successes. We then compare the empirical success rate to $0.95$ and proceed with our binary search accordingly. For a specific value of $p$, we estimate the corresponding sample complexity, by performing $50$ independent binary searches and computing the average value. The results, which are shown in Figure \ref{fig:SCrandom}, indicate that the dependence of sample complexity on the frequency parameter $p$ is indeed $\Theta(1/p)$.

\begin{figure}[ht]
    \centering
    \includegraphics[width=.75\linewidth]{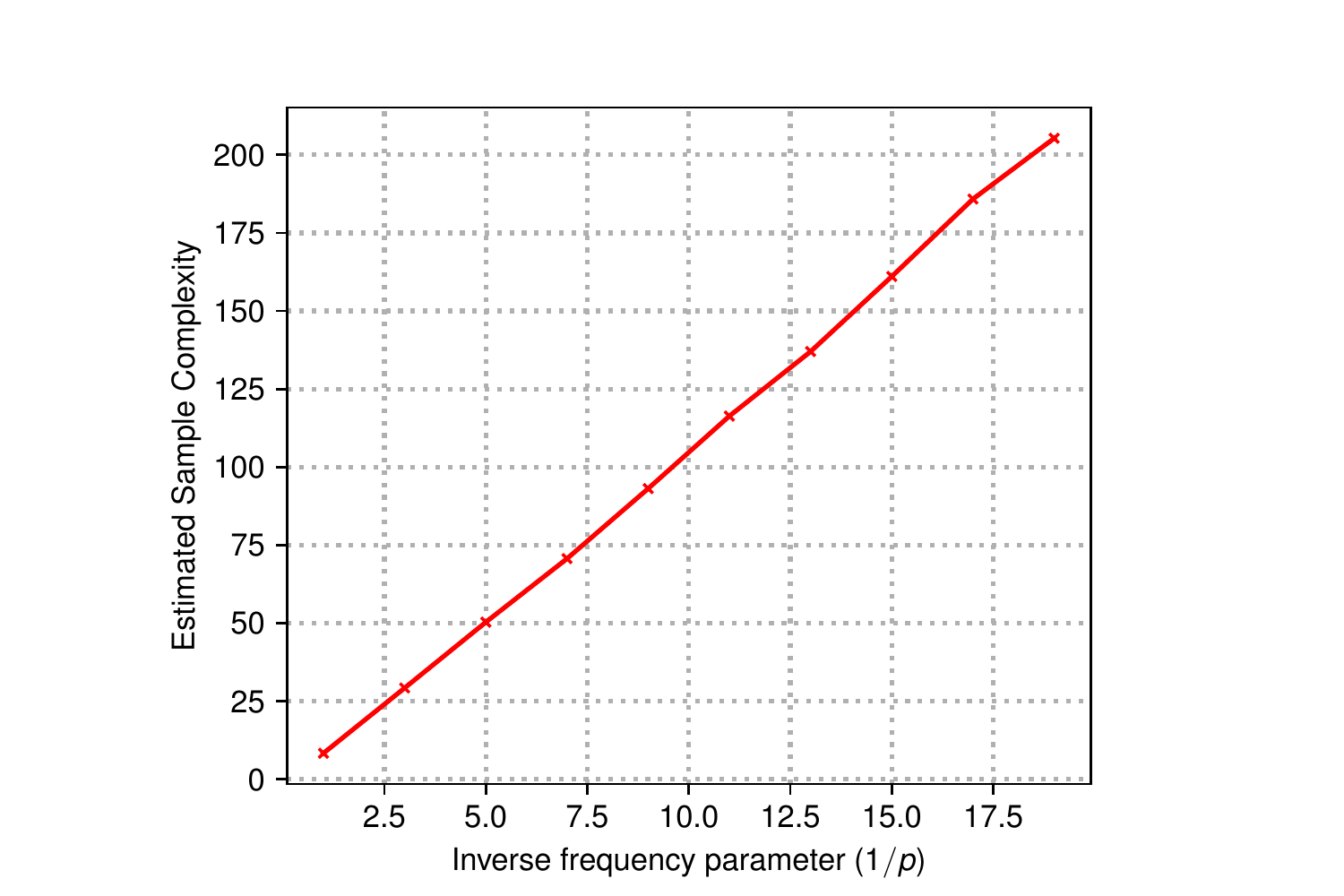}
    \caption{Estimated sample complexity of retrieving, with probability at least $0.95$ and using $\posest$, the central ranking from selective Mallows samples, with $n=20$, $\beta=2$, over the randomized frequency parameter's inverse.}\label{fig:SCrandom}
\end{figure}

\end{document}